\documentclass{article} 
\usepackage{hyperref}
\usepackage{url}

\usepackage[margin=1in]{geometry}
\usepackage{times}
\usepackage{amsthm,amsmath,amssymb,graphicx,url}%
\usepackage{cite}
\usepackage{color}
\usepackage{amsfonts}
\usepackage{enumerate,enumitem}
\usepackage{algorithm,algpseudocode}
\usepackage{epstopdf}
\usepackage{multirow}
\usepackage{enumitem}

\usepackage{xr}
\usepackage{subfigure}

\title{Learning Causal Graphs with Small Interventions}

\author{
Karthikeyan Shanmugam$^{1}$, Murat Kocaoglu$^{2}$,  Alexandros G. Dimakis$^{3}$, Sriram Vishwanath$^{4}$\\
Department of Electrical and Computer Engineering\\
The University of Texas at Austin, USA\\
$^{1}$\texttt{karthiksh@utexas.edu}, $^{2}$\texttt{mkocaoglu@utexas.edu} \\
$^{3}$\texttt{dimakis@austin.utexas.edu},$^{4} \texttt{sriram@ece.utexas.edu}$}

\newtheorem{definition}{Definition}
\newtheorem{lemma}{Lemma}

\newtheorem{problem}{Problem}
\newtheorem{theorem}{Theorem}

\newcommand{\ess}{\mathcal{E}}

\begin{document}

\maketitle

\begin{abstract}
We consider the problem of learning causal networks with interventions, when each intervention is limited in size under Pearl's Structural Equation Model with independent errors (SEM-IE). The objective is to minimize the number of experiments to discover the causal directions of all the edges in a causal graph. Previous work has focused on the use of separating systems for complete graphs for this task. We prove that any deterministic adaptive algorithm needs to be a separating system in order to learn complete graphs in the worst case. In addition, we present a novel separating system construction, whose size is close to optimal and is arguably simpler than previous work in combinatorics. We also develop a novel information theoretic lower bound on the number of interventions that applies in full generality, including for randomized adaptive learning algorithms. 

For general chordal graphs, we derive worst case lower bounds on the number of interventions. Building on observations about induced trees, we give a new deterministic adaptive algorithm to learn directions on any chordal skeleton completely. In the worst case, our achievable scheme is an $\alpha$-approximation algorithm where $\alpha$ is the independence number of the graph.  We also show that there exist graph classes for which the sufficient number of experiments is close to the lower bound. In the other extreme, there are graph classes for which the required number of experiments is multiplicatively $\alpha$ away from our lower bound. 

In simulations, our algorithm almost always performs very close to the lower bound, while the approach based on separating systems for complete graphs is significantly worse for random chordal graphs.
%
\end{abstract}

\section{Introduction}
Causality is a fundamental concept in sciences and philosophy.  The mathematical formulation of a \emph{theory of causality} in a probabilistic sense has received significant attention recently (e.g. \cite{Pearl2009,hauser2014two,Eberhardt2005,Hyttinen2013,Hu2014}). A formulation advocated by Pearl considers the \emph{structural equation models}: In this framework, $X$ is a cause of $Y$, if $Y$ can be written as $f(X,E)$, for some deterministic function $f$ and some latent random variable $E$. Given two causally related variables $X$ and $Y$, it is not possible to infer whether $X$ \emph{causes} $Y$ or $Y$ \emph{causes} $X$ from random samples, unless certain assumptions are made on the distribution of $E$ and/or on $f$ \cite{Shimizu2006,Hoyer2008}. 
For more than two random variables, \emph{directed acyclic graphs} (DAGs) are the most common tool used for representing causal relations. 
For a given DAG $D=(V,E)$, the directed edge $(X,Y)\in E$ shows that $X$ is a cause of $Y$. 

If we make no assumptions on the data generating process, the standard way of inferring the causal directions is by performing experiments, the so-called \emph{interventions}. An intervention requires modifying the process that generates the random variables: The experimenter has to enforce values on the random variables. This process is different than conditioning as explained in detail in \cite{Pearl2009}. 

The natural problem to consider is therefore minimizing the \textit{number} of interventions required to learn a causal DAG.  
Hauser et al.~\cite{hauser2014two} developed an efficient algorithm that minimizes this number in the worst case. The algorithm is based on optimal coloring of chordal graphs and requires at most $\log \chi$ interventions to learn any causal graph where $\chi$ is the chromatic number of the chordal skeleton. 

However, one important open problem appears when one also considers the \textit{size} of the used interventions: Each intervention is an experiment where the scientist must force a set of variables to take random values. Unfortunately, the interventions obtained in 
\cite{hauser2014two} can involve up to $n/2$ variables. 
The simultaneous enforcing of many variables can be quite challenging in many applications: for example in biology, some variables may not be enforceable at all or may require complicated genomic interventions for each parameter.

In this paper, we consider the problem of learning a causal graph when intervention sizes are bounded by some parameter $k$.
The first work we are aware of for this problem is by Eberhardt et al. ~\cite{Eberhardt2005}, where he provided an achievable scheme.
Furthermore~\cite{EberhardtThesis} shows that the set of interventions to fully identify a causal DAG must satisfy a specific set of combinatorial conditions called \textit{a separating system}\footnote{A separating system is a $0$-$1$ matrix with $n$ distinct columns and each row has at most $k$ ones.}, when the intervention size is not constrained or is 1. In \cite{Hyttinen2013}, with the assumption that the same holds true for any intervention size, Hyttinen et al. draw connections between causality and known separating system constructions. 
One open problem is: If the learning algorithm is \textit{adaptive} after each intervention, is a separating system still needed or can one do better? It was believed that adaptivity does not help in the worst case \cite{EberhardtThesis} and that one still needs a separating system. 

\textbf{Our Contributions:}
We obtain several novel results for learning causal graphs with interventions bounded by size $k$.
The problem can be separated for the special case where the underlying undirected graph (the skeleton) is the complete graph and the more general case where the underlying undirected graph is chordal. 
 \begin{enumerate}[topsep=0pt,itemsep=-1ex,partopsep=1ex,parsep=1ex,leftmargin=1.4em]
    \item For complete graph skeletons, we show that any adaptive deterministic algorithm needs a $(n,k)$ separating system.  This implies that lower bounds for separating systems also hold for adaptive algorithms and resolves the previously mentioned open problem. 
    \item 
	We present a novel combinatorial construction of a separating system that is close to the previous lower bound. This simple construction may be of more general interest in combinatorics. 
	\item Recently~\cite{Hu2014} showed that \textit{randomized} adaptive algorithms need only $\log \log n$ interventions with high probability for the unbounded case.  We extend this result and show that $O \left(\frac{n}{k} \log \log k \right)$ interventions of size bounded by $k$ suffice with high probability. 
    \item We present a more general information theoretic lower bound of $\frac{n}{2k}$ to capture the performance of such randomized algorithms. 
    \item We extend the lower bound for adaptive algorithms for general chordal graphs. We show that over all orientations, the number of experiments from a $(\chi(G),k)$ separating system is needed where $\chi(G)$ is the chromatic number of the skeleton graph. 
    \item We show two extremal classes of graphs. For one of them, the interventions through $(\chi,k)$ separating system is sufficient. For the other class, we need $\frac{\alpha \left(\chi-1 \right)}{2k} \approx \frac{n}{2k}$ experiments in the worst case.
    \item We exploit the structural properties of chordal graphs to design a new deterministic adaptive algorithm that uses the idea of separating systems together with adaptability to Meek rules. 
We simulate our new algorithm and empirically observe that it performs quite close to the $(\chi,k)$ separating system. Our algorithm requires much fewer interventions compared to $(n,k)$ separating systems.
 \end{enumerate}

\section{Background and Terminology}
 \subsection{Essential graphs}
A \emph{causal DAG} $D=(V,E)$ is a directed acyclic graph where $V=\{x_1,x_2 \ldots x_n\}$ is a set of random variables and $(x,y) \in E$ is a directed edge if and only if $x$ is a direct \emph {cause} of $y$. We adopt Pearl's \textit{structural equation model with independent errors} (SEM-IE) in this work (see \cite{Pearl2009} for more details). Variables in $S \subseteq V$  \textit{cause} $x_i$, if $x_i = f(\{x_j\}_{j \in S}, e_y )$ where $e_y$ is a random variable independent of all other variables. 

The causal relations of $D$ imply a set of conditional independence (CI) relations between the variables. A conditional independence relation is of the following form: Given $Z$, the set $X$ and the set $Y$ are conditionally independent for some disjoint subsets of variables $X,Y,Z$. Due to this, causal DAGs are also called \emph{causal Bayesian networks}. A set $V$ of variables is Bayesian with respect to a DAG $D$ if the joint probability distribution of $V$ can be factorized as a product of marginals of every variable conditioned on its parents. 

All the CI relations that are learned statistically through observations can also be inferred from the Bayesian network using a graphical criterion called the \emph{d-separation} \cite{Spirtes2001} assuming that the distribution is faithful to the graph \footnote{Given Bayesian network, any CI relation implied by d-separation holds true. All the CIs implied by the distribution can be found using d-separation if the distribution is faithful. Faithfulness is a widely accepted assumption, since it is known that only a measure zero set of distributions are not faithful \cite{Meek1995b}.}. Two causal DAGs are said to be \emph{Markov equivalent} if they encode the same set of CIs. Two causal DAGs are Markov equivalent if and only if they have the same skeleton\footnote{Skeleton of a DAG is the undirected graph obtained when directed edges are converted to undirected edges.} and the same immoralities\footnote{An induced subgraph on $X,Y,Z$ is an immorality if $X$ and $Y$ are disconnected, $X\rightarrow Z$ and $Z\leftarrow Y$. }. The class of causal DAGs that encode the same set of CIs is called the \emph{Markov equivalence class}. 
We denote the Markov equivalence class of a DAG $D$ by $[D]$.  The graph union\footnote{Graph union of two DAGs $D_1=(V,E_1)$ and $D_2=(V,E_2)$ with the same skeleton is a partially directed graph $D=(V,E)$, where $(v_a,v_b)\in E$ is undirected if the edges $(v_a,v_b)$ in $E_1$ and $E_2$ have different directions, and directed as $v_a\rightarrow v_b$ if the edges $(v_a,v_b)$ in $E_1$ and $E_2$ are both directed as $v_a\rightarrow v_b$.} of all DAGs in $[D]$ is called the \emph{essential graph} of $D$. It is denoted $\ess(D)$. 
 $\ess(D)$ is always a chain graph 
 with chordal\footnote{An undirected graph is chordal if it has no induced cycle of length greater than $3$.} chain components \footnote{This means that $\ess(D)$ can be decomposed as a sequence of undirected chordal graphs $G_1,G_2 \ldots G_m$ (chain components) such that there is a directed edge from a vertex in $G_i$ to a vertex in $G_j$ only if $i<j$ } \cite{Andersson1997}.

The $d$-separation criterion can be used to identify the skeleton and all the immoralities of the underlying causal DAG \cite{Spirtes2001}. Additional edges can be identified using the fact that the underlying DAG is acyclic and there are no more immoralities. Meek derived $3$ local rules (\emph{Meek rules}), introduced in \cite{verma1992algorithm}, to be recursively applied to identify every such additional edge (see Theorem 3 of \cite{Meek1995a}). The repeated application of \emph{Meek rules} on this partially directed graph with identified immoralities until they can no longer be used yields the essential graph.

\subsection{Interventions and Active Learning}
Given a set of variables $V=\{x_1,...,x_n\}$, an \emph{intervention} on a set $S \subset X$ of the variables is an experiment where the performer forces each variable $s \in S$ to take the value of another independent (from other variables) variable $u$, i.e., $s=u$. This operation, and how it affects the joint distribution is formalized by the \emph{do} operator by Pearl \cite{Pearl2009}. 
 An intervention modifies the causal DAG $D$ as follows: The post intervention DAG $D_{\{S\}}$ is obtained by removing the connections of nodes in $S$ to their parents. 
 The \emph{size of an intervention } $S$ is the number of intervened variables, i.e., $|S|$. Let $S^c$ denote the complement of the set $S$.
 
 CI-based learning algorithms can be applied to $D_{\{S\}}$ to identify the set of removed edges, i.e. parents of $S$ \cite{Spirtes2001}, and the remaining adjacent edges in the original skeleton are declared to be the children. Hence,
    
 (R0) The orientations of the edges of the cut between $S$ and $S^c$ in the original DAG $D$ can be inferred. 
    
    Then, $4$ local Meek rules (introduced in \cite{verma1992algorithm}) are repeatedly applied to the original DAG $D$ with the new directions learnt from the cut to learn more till no more directed edges can be identified. Further application of CI-based algorithms on $D$ will reveal no more information. The Meek rules are given below:
\begin{itemize}[topsep=0pt,itemsep=-1ex,partopsep=1ex,parsep=1ex]
\item[(R1)] ($a-b$) is oriented as ($a\rightarrow b$) if $\exists c$ s.t. $(c\rightarrow a)$ and $(c,b)\notin E$. 
\item[(R2)] ($a-b$) is oriented as ($a\rightarrow b$) if $\exists c$ s.t. $(a\rightarrow c)$ and $(c\rightarrow b)$.
\item[(R3)] ($a-b$) is oriented as ($a\rightarrow b$) if $\exists c,d$ s.t. $(a-c)$,$(a-d)$,$(c\rightarrow b)$,$(d\rightarrow b)$ and $(c,d)\notin E$.
\item[(R4)] ($a-c$) is oriented as ($a\rightarrow c$) if $\exists b,d$ s.t. $(b\rightarrow c)$,$(a-d)$,$(a-b)$,$(d\rightarrow b)$ and $(c,d)\notin E$.
\end{itemize}
The concepts of essential graphs and Markov equivalence classes are extended in \cite{Hauser2012a} to incorporate the role of interventions: Let $\mathcal{I}=\{I_1,I_2,...,I_m\}$, be a set of interventions and let the above process be followed after each intervention. Interventional Markov equivalence class ($\mathcal{I}$ equivalence) of a DAG is the set of DAGs that represent the same set of probability distributions obtained when the above process is applied after every intervention in $\mathcal{I}$. It is denoted by $[D]_{\mathcal{I}}$. 
Similar to the observational case, \emph{$\mathcal{I}$ essential graph} of a DAG $D$ is the graph union of all DAGs in the same $\mathcal{I}$ equivalence class; it is denoted by $\ess_{\mathcal{I}}(D)$. 
 We have the following sequence:
   \begin{align}\label{eqn:learnseq}
     D   \rightarrow \mathrm{CI~learning} \rightarrow \mathrm{Meek~rules}  \rightarrow \ess(D)  \rightarrow I_1 \overset{a}\rightarrow \mathrm{learn~by~R0} \overset{b}\rightarrow 
   \mathrm{Meek~rules} & \nonumber \\
      \rightarrow \ess_{\{I_1\}} (D) \rightarrow I_2 \ldots \rightarrow \ess_{\{I_1,I_2\}} (D) \ldots &
   \end{align}

 Therefore, after a set of interventions ${\cal I}$ has been performed, the essential graph 
 $\ess_{{\cal I}}(D)$ is a graph with some oriented edges that captures all the causal relations we have discovered so far, using 
 ${\cal I}$. Before any interventions happened $\ess (D)$ captures the initially known causal directions. 
 It is known that $\ess_{{\cal I}}(D)$ is a chain graph with chordal chain components. Therefore when all the directed edges are removed, the graph becomes a set of disjoint chordal graphs.

\subsection{Problem Definition}


We are interested in the following question: 
\begin{problem}
Given that all interventions in ${\cal I}$ are of size at most $k < n/2$ variables, i.e., 
for each intervention $I$, $\lvert I \rvert \leq k,\forall I\in\mathcal{I}$, minimize the number of interventions $\lvert {\cal I} \rvert$ such that the partially directed graph with all directions learned so far $\ess_{{\cal I}} (D) =D$. 
\end{problem}

The question is the design of an algorithm that computes the small set of interventions ${\cal I}$ given $\ess(D)$. Note, of course, that the unknown directions of the edges $D$ are not available to the algorithm. One can view the design of ${\cal I}$ as an active learning process to find $D$ from the essential graph $\ess(D)$. $\ess(D)$ is a chain graph with undirected chordal components and it is known that interventions on one chain components do not affect the discovery process of directed edges in the other components \cite{Hauser2012b}. So we will assume that $\ess(D)$ is undirected and a chordal graph to start with. Our notion of algorithm does not consider the time complexity (of statistical algorithms involved) of steps $a$ and $b$ in (\ref{eqn:learnseq}). Given $m$ interventions, we only consider efficiently computing $I_{m+1}$ using (possibly) the graph $\ess_{\{I_{1}, \ldots I_{m} \}}$. We consider the following three classes of algorithms:
\begin{enumerate}[topsep=0pt,itemsep=-1ex,partopsep=1ex,parsep=1ex]
 \item \textit{Non-adaptive algorithm:} The choice of ${\cal I}$ is fixed prior to the discovery process.
 \item  \textit{Adaptive algorithm:} At every step $m$, the choice of $I_{m+1}$ is a deterministic function of $\ess_{\{I_{1}, \ldots I_{m} \}}(D)$.
 \item \textit{Randomized adaptive algorithm:} At every step $m$,  the choice of $I_{m+1}$ is a random function of $\ess_{\{I_{1}, \ldots I_{m} \}}(D)$.
\end{enumerate}

The problem is different for complete graphs versus more general chordal graphs since rule R$1$ becomes applicable when the graph is not complete. Thus we give a separate treatment for each case. First, we provide algorithms for all three cases for learning the directions of complete graphs $\ess(D)= K_n$ (undirected complete graph) on $n$ vertices. Then, we generalize to chordal graph skeletons and provide a novel adaptive algorithm with upper and lower bounds on its performance. 

The missing proofs of the results that follow can be found in the Appendix.


\section{Complete Graphs}
\label{sec:complete}
In this section, we consider the case where the skeleton we start with, i.e. $\ess(D)$, is an undirected complete graph (denoted $K_n$). It is known that at any stage in (\ref{eqn:learnseq}) starting from $\ess(D)$, rules R$1$, R$3$ and R$4$ do not apply. Further, the underlying DAG $D$ is a directed clique. The directed clique is characterized by an ordering $\sigma$ on $[1:n]$ such that, in the subgraph induced by $\sigma(i),\sigma(i+1) \ldots \sigma(n)$, $\sigma(i)$ has no incoming edges. Let $D$ be denoted by $\vec{K}_n(\sigma)$ for some ordering $\sigma$. Let $[1:n]$ denote the set $\{1,2 \ldots n\}$. We need the following results on a separating system for our first result regarding adaptive and non-adaptive algorithms for a complete graph.

\subsection{Separating System}
\begin{definition}\cite{Katona1966,Wegener1979}
 An $(n,k)$-\textit{separating} system on an $n$ element set $[1:n]$ is a set of subsets ${\cal S}=\{S_1,S_2 \ldots S_m \}$ such that $\lvert S_i \rvert \leq k$ and for every pair $i,j$ there is a subset $S \in {\cal S}$ such that either $i \in S,~ j \notin S$ or $j \in S,~ i \notin S$. If a pair $i,j$ satisfies the above condition with respect to ${\cal S}$, then ${\cal S}$ is said to separate the pair $i,j$. Here, we consider the case when $k<n/2$
\end{definition}

In \cite{Katona1966}, Katona gave an $(n,k)$-separating system together with a lower bound on $\lvert \cal S\rvert$. In \cite{Wegener1979}, Wegener gave a simpler argument for the lower bound and also provided a tighter upper bound than the one in \cite{Katona1966}. In this work, we give a different construction below where the separating system size is at most$\lceil \log_{\lceil n/k \rceil} n\rceil$ larger than the construction of Wegener. However, our construction has a simpler description. 

\begin{lemma}\label{label}
   There is a labeling procedure that produces distinct $\ell$ length labels for all elements in $[1:n]$ using letters from the integer alphabet $\{0,1 \ldots a \}$ where $\ell=\lceil \log_{a} n \rceil$. Further, in every digit (or position), any integer letter is used at most $\lceil n/a \rceil$ times.
\end{lemma}

Once we have a set of $n$ string labels as in Lemma \ref{label}, our separating system construction is straightforward.
\begin{theorem}\label{const}
 Consider an alphabet ${\cal A}=[0:\lceil \frac{n}{k} \rceil]$ of size $\lceil \frac{n}{k} \rceil+1$ where $k<n/2$. Label every element of an $n$ element set using a distinct string of letters from ${\cal A}$ of length $ \ell= \lceil \log_{\lceil \frac{n}{k} \rceil} n \rceil$ using the procedure in Lemma \ref{label} with $a=\lceil \frac{n}{k} \rceil$. For every $1 \leq i \leq \ell$ and $1 \leq j \leq \lceil \frac{n}{k} \rceil $, choose the subset $S_{i,j}$ of vertices whose string's $i$-th letter is $j$. The set of all such subsets ${\cal S} = \{S_{i,j}\}$ is a $k$-separating system on $n$ elements and $\lvert {\cal S} \rvert \leq (\lceil \frac{n}{k} \rceil ) \lceil \log_{\lceil \frac{n}{k} \rceil} n \rceil $.
\end{theorem}

\subsection{Adaptive algorithms: Equivalence to a Separating System}
  Consider any non-adaptive algorithm that designs a set of interventions ${\cal I}$, each of size at most $k$, to discover $\vec{K}_n(\sigma)$. ${\cal I}$ has to be a separating system in the worst case over all $\sigma$. This is already known. Now, we prove the necessity of a separating system for deterministic adaptive algorithms in the worst case.
  
  \begin{theorem}\label{equivsep}
     Let there be an adaptive deterministic algorithm $A$ that designs the set of interventions ${\cal I}$ such that the final graph learnt $\ess_{{\cal I}}(D)= \vec{K}_n(\sigma)$ for any ground truth ordering $\sigma$ starting from the initial skeleton $\ess(D) =K_n$. Then, there exists a $\sigma$ such that $A$ designs an ${\cal I}$ which is a separating system.
  \end{theorem}
  The theorem above is independent of the individual intervention sizes. Therefore, we have the following theorem, which is a direct corollary of Theorem \ref{equivsep}:   
  \begin{theorem}
   In the worst case over $\sigma$, any adaptive or a non-adaptive deterministic algorithm on the DAG $\vec{K}_n(\sigma)$ has to be such that $\frac{n}{k}  \log_{\frac{ne}{k}} n \leq \lvert {\cal I} \rvert$. There is a feasible ${\cal I}$ with $\lvert {\cal I} \rvert \leq \lceil (\frac{n}{k} \rceil-1)  \lceil \log_{\lceil \frac{n}{k} \rceil} n \rceil$
  \end{theorem}
  \begin{proof}
       By Theorem \ref{equivsep}, we need a separating system in the worst case and the lower and upper bounds are from \cite{Wegener1979,Katona1966}.
  \end{proof}
  
\subsection{Randomized Adaptive Algorithms}
In this section, we show that that total number of variable accesses to fully identify the complete causal DAG is $\Omega(n)$.
\begin{theorem}
\label{thm:infoLower}
To fully identify a complete causal DAG $\vec{K}_n(\sigma)$ on $n$ variables using size-$k$ interventions, $\frac{n}{2k}$ interventions are necessary. Also, the total number of variables accessed is at least $\frac{n}{2}$.
\end{theorem}

The lower bound in Theorem \ref{thm:infoLower} is information theoretic. We now give a randomized algorithm that requires $O(\frac{n}{k}\log\log k)$ experiments in expectation. We provide a straightforward generalization of \cite{Hu2014}, where the authors gave a randomized algorithm for unbounded intervention size. 
\begin{theorem} \label{thm:randomized}
Let $\ess(D)$ be $K_n$ and the experiment size $k=n^r$ for some $0<r<1$. Then there exists a randomized adaptive algorithm which designs an $\mathcal{I}$ such that $\ess_{\mathcal{I}}(D)=D$ with probability polynomial in $n$, and $\lvert\mathcal{I}\rvert=\mathcal{O}(\frac{n}{k}\log\log(k))$ in expectation.
\end{theorem}

%
%
%
%



\section{General Chordal Graphs}
\label{sec:chordal}

In this section, we turn to interventions on a general DAG $G$. After the initial stages in (\ref{eqn:learnseq}), $\ess(G)$ is a chain graph with chordal chain components. There are no further immoralities throughout the graph. In this work, we focus on one of the chordal chain components. Thus the DAG $D$ we work on is assumed to be a directed graph with no immoralities and whose skeleton $\ess(D)$ is chordal. We are interested in recovering $D$ from $\ess(D)$ using interventions of size at most $k$ following (\ref{eqn:learnseq}).

\subsection{Bounds for Chordal skeletons}
We provide a lower bound for both adaptive and non-adaptive deterministic schemes for a chordal skeleton $\ess(D)$. Let $\chi \left( \ess(D) \right)$ be the coloring number of the given chordal graph. Since, chordal graphs are perfect, it is the same as the clique number.

\begin{theorem}\label{lowbnd}
 Given a chordal $\ess(D)$, in the worst case over all DAGs $D$ (which has skeleton $\ess(D)$ and no immoralities), if every intervention is of size at most $k$, then $ \lvert {\cal I}\rvert \geq \frac{\chi \left(\ess(D)\right)}{k} \log_{\frac{\chi \left(\ess(D)\right)e}{k}} \chi \left(\ess(D)\right)$ for any adaptive and non-adaptive algorithm with $\ess_{\mathcal{I}}(D)=D$.
\end{theorem}

\textit{Upper bound:} Clearly, the separating system based algorithm of Section \ref{sec:complete} can be applied to the vertices in the chordal skeleton $\ess(D)$ and it is possible to find all the directions. Thus, $ \lvert {\cal I} \rvert \leq \frac{n}{k}  \log_{\lceil \frac{n}{k} \rceil} n \leq \frac{\alpha (\ess(D)) \chi(\ess(D)) }{k}  \log_{\lceil \frac{n}{k} \rceil} n$. This with the lower bound implies an $\alpha$ approximation algorithm (since $\log_{\lceil \frac{n}{k} \rceil} n \leq  \log_{\frac{\chi \left(\ess(D)\right)e}{k}} \chi \left(\ess(D)\right) $ , under a mild assumption $\chi(\ess(D)) \leq \frac{n}{e}$ ).

\textbf{Remark:} The separating system on $n$ nodes gives an $\alpha$ approximation. However, the new algorithm in Section \ref{new} exploits chordality and performs much better empirically. It is possible to show that our heuristic also has an $\alpha$ approximation guarantee but we skip that.
\subsection{Two extreme counter examples}
   We provide two classes of chordal skeletons $G$: One for which the number of interventions close to the lower bound is \emph{sufficient} and the other for which the number of interventions \emph{needed} is very close to the upper bound.
   
 \begin{theorem}\label{lower-bound} 
  There exists chordal skeletons such that for any algorithm with intervention size constraint $k$, the number of interventions $\lvert \cal I \rvert$ required is at least $\alpha \frac{(\chi-1)}{2k}$ where $\alpha$ and $\chi$ are the independence number and chromatic numbers respectively. There exists chordal graph classes such that $\lvert {\cal I} \rvert = \lceil \frac{\chi}{k} \rceil  \lceil \log_{\lceil \frac{\chi}{k} \rceil} \chi \rceil$ is sufficient.
 \end{theorem}
 
\subsection{An Improved Algorithm using Meek Rules}\label{new}

 In this section, we design an adaptive deterministic algorithm that \textit{anticipates} Meek rule R$1$ usage along with the idea of a separating system. We evaluate this experimentally on random chordal graphs. First, we make a few observations on learning connected directed trees $T$ from the skeleton $\ess(T)$ (undirected trees are chordal) that do not have immoralities using Meek rule R$1$ where every intervention is of size $k=1$. Because the tree has no cycle, Meek rules R$2$-R$4$ do not apply.  
 \begin{lemma}\label{lem:root}
    Every node in a directed tree with no immoralities has at most one incoming edge. There is a \textit{root} node with no incoming edges and intervening on that node alone identifies the whole tree using repeated application of rule R$1$.
 \end{lemma}
 \begin{lemma}\label{treealg}
  If every intervention in ${\cal I}$ is of size at most $1$, learning all directions on a directed tree $T$ with no immoralities can be done adaptively with at most $\lvert {\cal I} \rvert \leq O(\log_2 n)$ where $n$ is the number of vertices in the tree. The algorithm runs in time $\mathrm{poly}(n)$.
 \end{lemma}
 \begin{lemma}\label{inductree}
  Given any chordal graph and a valid coloring, the graph induced by any two color classes is a forest.
 \end{lemma}
  
  In the next section, we combine the above single intervention adaptive algorithm on directed trees which uses Meek rules, with that of the non-adaptive separating system approach.

 \subsubsection{Description of the algorithm}\label{ideas}
 
 The key motivation behind the algorithm is that, a pair of color classes is a forest (Lemma \ref{inductree}). Choosing the right node to intervene leaves only a small subtree unlearnt as in the proof of Lemma \ref{treealg}. In subsequent steps, suitable nodes in the remaining subtrees could be chosen until all edges are learnt. We give a brief description of the algorithm below.
 
Let $G$ denote the initial undirected chordal skeleton $\ess(D)$ and let $\chi$ be its coloring number. Consider a $(\chi,k)$ separating system ${\cal S}=\{S_i\}$. To intervene on the actual graph, an intervention set $I_i$ corresponding to $S_i$ is chosen. We would like to intervene on a node of color $c \in S_i$.
 
 Consider a node $v$ of color $c$. Now, we attach a score $P(v,c)$ as follows. For any color $c'\notin S_i$, consider the induced forest $F(c,c')$ on the color classes $c$ and $c'$ in $G$. Consider the tree $T(v,c,c')$ containing node $v$ in $F$. Let $d(v)$ be the degree of $v$ in $T$. Let $T_1, T_2, \ldots T_{d(v)}$ be the resulting disjoint trees after node $v$ is removed from $T$. If $v$ is intervened on, according to the proof of Lemma \ref{treealg}: a)
 All edge directions in all trees $T_i$ except one of them would be learnt when applying Meek Rules and rule R$0$. b) All the directions from $v$ to all its neighbors would be found. 

The score is taken to be the total number of edge directions guaranteed to be learnt in the worst case. 
Therefore, the score $P(v)$ is:
$
   P(v) = \sum \limits_{c': \lvert {c,c'} \bigcap \rvert =1} \left( \lvert T(c,c') \rvert - \max \limits_{ 1 \leq j \leq d(v)} \lvert T_{j} \rvert \right).
$
 The node with the highest score among the color class $c$ is used for the intervention $I_i$. After intervening on $I_i$, all the edges whose directions are known through Meek Rules (by repeated application till nothing more can be learnt) and R$0$ are deleted from $G$. Once ${\cal S}$ is processed, we recolor the sparser graph $G$. We find a new ${\cal S}$ with the new chromatic number on $G$ and the above procedure is repeated. The exact hybrid algorithm is described in Algorithm \ref{alg:hybrid}. 

\begin{theorem}\label{Correctness}
 Given an undirected choral skeleton $G$ of an underlying directed graph with no immoralities, Algorithm \ref{alg:hybrid} ends in finite time and it returns the correct underlying directed graph. The algorithm has runtime complexity polynomial in $n$.
\end{theorem}

\begin{algorithm}[ht!]
    \caption{Hybrid Algorithm using Meek rules with separating system}
   \label{alg:hybrid}
\begin{algorithmic}[1]
    \State {\bfseries Input:} Chordal Graph skeleton $G=(V,E)$ with no Immoralities.
    \State Initialize $\vec{G}(V,E_d=\emptyset)$ with $n$  nodes and no directed edges. Initialize time $t=1$.
     \While  {$E \neq \emptyset$} 
     \State \label{linecolor} Color the chordal graph $G$ with $\chi$ colors.  \Comment{Standard algorithms exist to do it in linear time}
     \State Initialize color set ${\cal C}= \{1,2 \ldots \chi\}$. Form a $(\chi, \min (k,\lceil \chi/2 \rceil ))$ separating system ${\cal S}$ such that $\lvert S \rvert \leq k,~ \forall S \in {\cal S}$.
       \For {$i=1$ until $\lvert {\cal S} \rvert $}
         \State Initialize Intervention $I_t=\emptyset$.
             \For {$c \in S_i$ and every node $v$ in color class $c$}      
                  \State Consider $F(c,c')$, $T(c,c',v)$ and $\{T_j\}_{1}^{d(i)}$(as per definitions in Sec. \ref{ideas}). 
                 \State  Compute: $P(v,c) = \sum \limits_{c' \in {\cal C} \bigcap S_i^c} \lvert T(c,c',v) \rvert - \max \limits_{ 1 \leq j \leq d(i)} \lvert T_{j} \rvert $.         
          \EndFor
          \If {$k \leq \chi/2$}    
             \State $ I_t = I_t \bigcup \limits_{c \in S_i} \{\mathop{\rm argmax} \limits_{v:P(v,c) \neq 0} P(v,c)\}$.
           \Else
              \State $ I_t = I_t \mathop{\cup}_{c \in S_i} \{\mathrm{First~}\frac{k}{\lceil \chi/2 \rceil} \mathrm{~nodes~}v\mathrm{~with~ largest~ nonzero~} P(v,c) \}$.
           \EndIf
           \State $t=t+1$
           \State Apply R$0$ and Meek rules using $E_d$ and $E$ after intervention $I_t$. Add newly learnt directed edges to $E_d$ and delete them from $E$. 
        \EndFor
        \State Remove all nodes which have degree $0$ in G.        
    \EndWhile
    \State \Return $\vec{G}$.
\end{algorithmic}
\end{algorithm}

\section{Simulations}

We simulate our new heuristic, namely Algorithm \ref{alg:hybrid}, on randomly generated chordal graphs and compare it with a naive algorithm that follows the intervention sets given by our $(n,k)$ separating system as in Theorem \ref{const}.  Both algorithms apply R$0$ and Meek rules after each intervention according to (\ref{eqn:learnseq}). We plot the following lower bounds: a) \emph{Information Theoretic LB} of $\frac{\chi}{2k}$ b) \emph{Max. Clique Sep. Sys. Entropic LB} which is the chromatic number based lower bound of Theorem \ref{lowbnd}. Moreover, we use two known $(\chi,k)$ separating system constructions for the maximum clique size as ``references": The best known $(\chi,k)$ separating system is shown by the label \emph{Max. Clique Sep. Sys. Achievable LB} and our new simpler separating system construction (Theorem \ref{const}) is shown by \emph{Our Construction Clique Sep. Sys. LB}. As an upper bound, we use the size of the best known $(n,k)$ separating system (without any Meek rules) and is denoted \emph{Separating System UB}. 

\emph{Random generation of chordal graphs:} Start with a random ordering $\sigma$ on the vertices. Consider every vertex starting from $\sigma(n)$. For each vertex $i$, $(j,i)\in E$ with probability inversely proportional to $\sigma(i)$ for every $j\in S_i$ where $S_i=\{v:\sigma^{-1}(v)<\sigma^{-1}(i)\}$. The proportionality constant is changed to adjust sparsity of the graph. After all such $j$ are considered, make $S_i\cap \textrm{ne}(i)$ a clique by adding edges respecting the ordering $\sigma$, where $\textrm{ne}(i)$ is the neighborhood of $i$. The resultant graph is a DAG and the corresponding skeleton is chordal. Also, $\sigma$ is a perfect elimination ordering.

\begin{figure*}[ht!]
\centering
\subfigure[$n=1000,k=10$]{\label{fig:n1000k10}\includegraphics[width=0.45\columnwidth]{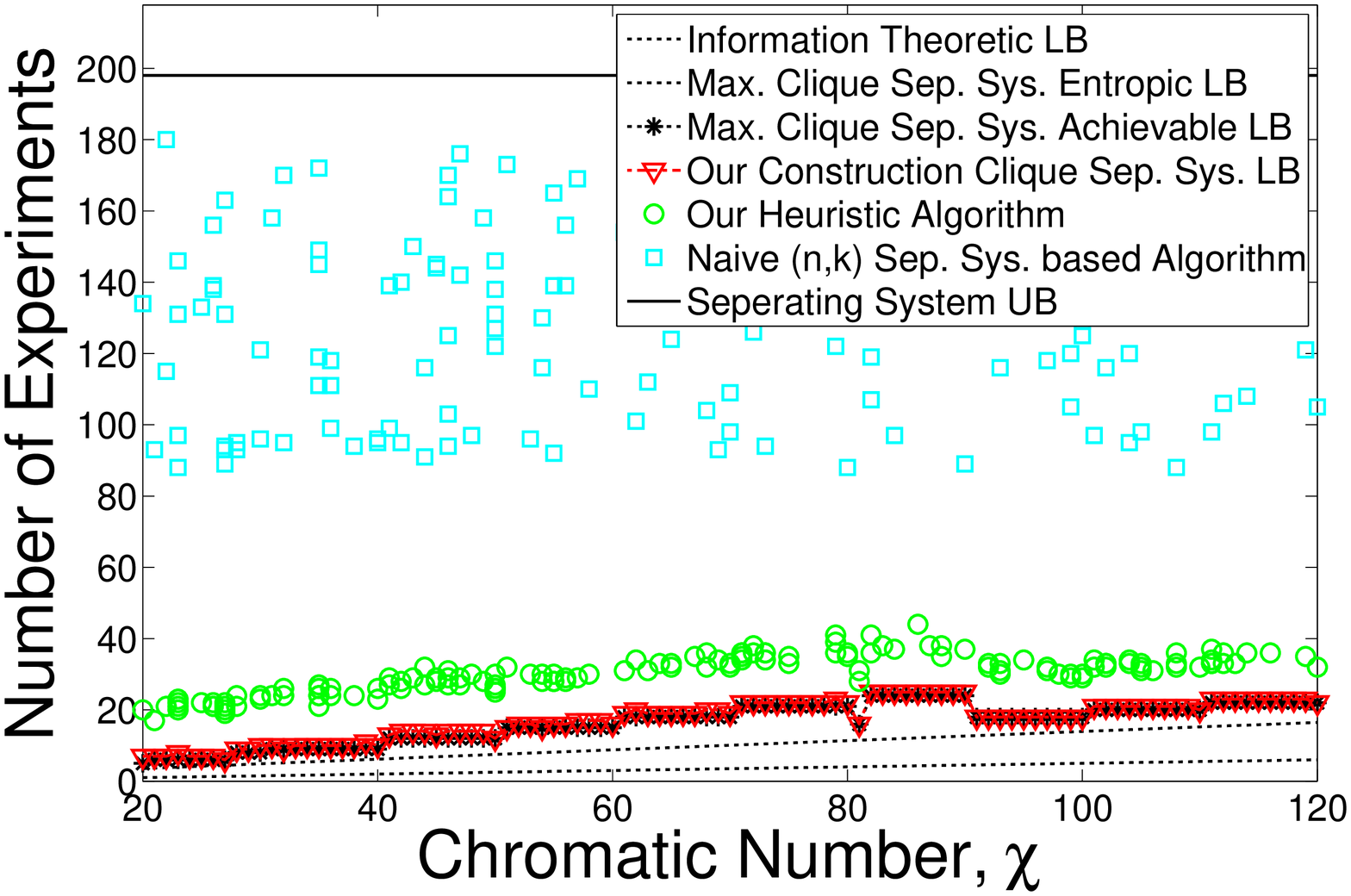}}
\subfigure[$n=2000,k=10$]{\label{fig:n2000k10}\includegraphics[width=0.45\columnwidth]{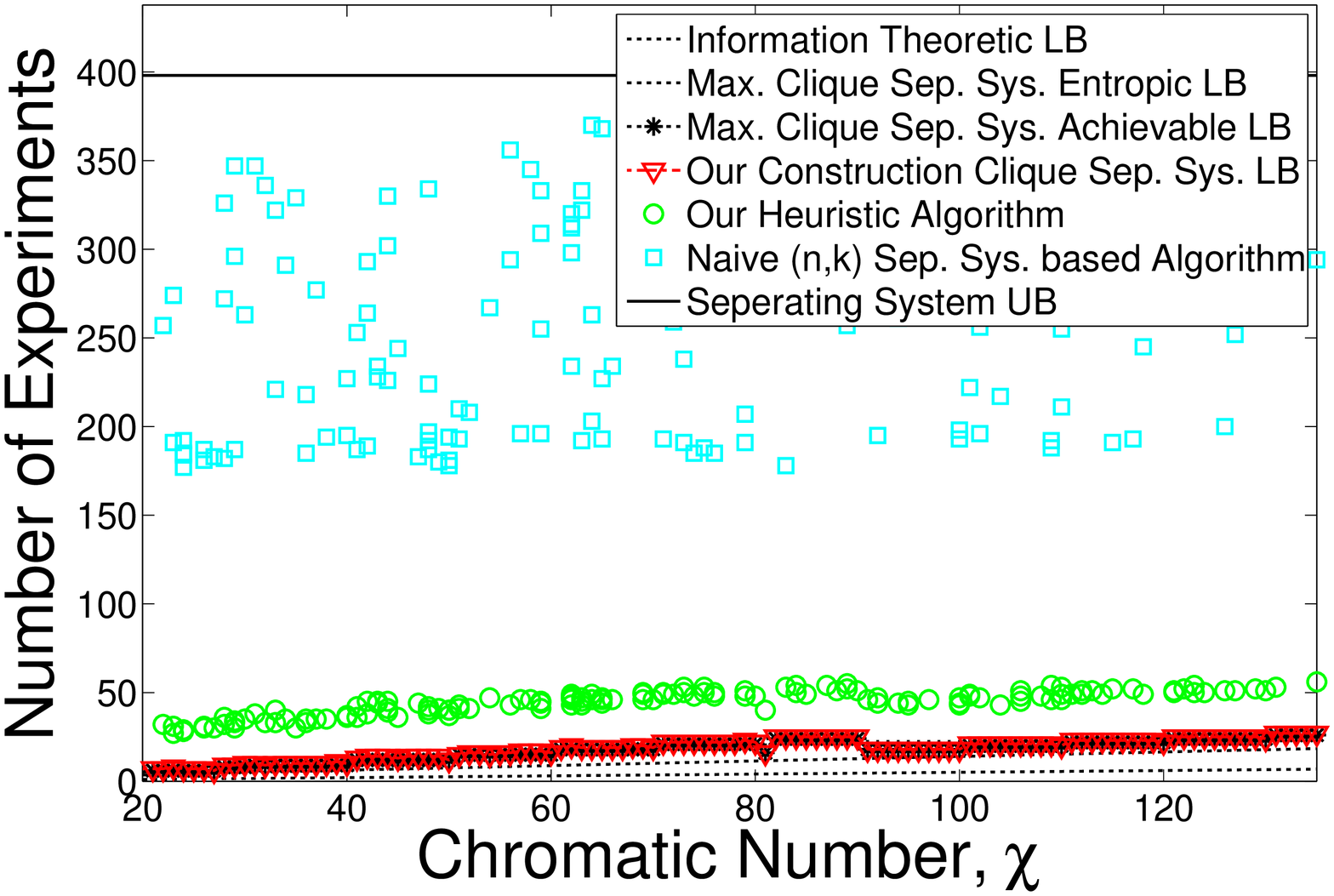}}
\caption{$n$: no. of vertices, $k$: Intervention size bound. The number of experiments is compared between our heuristic and the naive algorithm based on the $(n,k)$ separating system on random chordal graphs. The red markers represent the sizes of $(\chi,k)$ separating system. Green circle markers and the cyan square markers for the same $\chi$ value correspond to the number of experiments required by our heuristic and the algorithm based on an $(n,k)$ separating system(Theorem \ref{const}), respectively, on the same set of chordal graphs. Note that, when $n=1000$ and $n=2000$, the naive algorithm requires on average about $130$ and $260$ (close to $n/k$) experiments respectively, while our algorithm requires at most $\sim40$ (orderwise close to $\chi/k =10$) when $\chi=100$. 
}
\label{fig:simulation}
\end{figure*}

 \textbf{Results:} We are interested in comparing our algorithm and the naive one which depends on the $(n,k)$ separating system to the size of the $(\chi,k)$ separating system. The size of the $(\chi,k)$ separating system is roughly $\tilde{O}(\chi/k)$.  Consider values around $\chi=100$ on the x-axis for the plots with $n=1000,k=10$ and $n=2000,k=10$. Note that, our algorithm performs very close to the size of the $(\chi,k)$ separating system, i.e. $\tilde{O}(\chi/k)$. In fact, it is always $<40$ in both cases while the average performance of naive algorithm goes from $130$ (close to $n/k=100$) to $260$ (close to $n/k=200$). The result points to this: For random chordal graphs, the structured tree search allows us to learn the edges in a number of experiments quite close to the lower bound based only on the maximum clique size and not $n$. The plots for $(n,k)=(500,10)$ and $(n,k)=(2000,20)$ are given in Appendix.
 

\section{Conclusions}
  We have considered the problem of adaptively designing interventions of bounded size to learn a causal graph under Pearl's SEM-IE model. We proposed lower and upper bounds for the number of interventions needed in the worst case for various classes of algorithms, when the causal graph skeleton is complete. We developed lower and upper bounds on the minimum number of interventions required in the worst case for general graphs. We characterized two extremal graph classes such that the minimum number of interventions in one class is close to the lower bound and in the other class it is close to the upper bound. In the case of chordal skeletons, we proposed an algorithm that combines ideas for the complete graphs with the ones when the skeleton is a forest via application of Meek rules. Empirically, on randomly generated chordal graphs, our algorithm performs close to the lower bound and it outperforms the previous state of the art. Possible future work includes obtaining a tighter lower bound for chordal graphs that would possibly establish a tighter approximation guarantee for our algorithm.
\subsubsection*{Acknowledgments}
Authors acknowledge the support from grants: NSF CCF 1344179, 1344364, 1407278, 1422549 and a ARO YIP award (W911NF-14-1-0258). We also thank Frederick Eberhardt for helpful discussions.
\newpage
\bibliographystyle{IEEEtran}
\bibliography{ActiveLearning2015_Bib}

\newpage
\section*{Appendix}

 \subsection{Proof of Lemma \ref{label}}
 We describe a string labeling procedure as follows to label elements of the set $[1:n]$.
 
 \textit{String Labelling:}
  Let $a>1$ be a positive integer. Let $x$ be the integer such that $a^{x} < n \leq a^{x+1}$. $x+1 = \lceil \log_{a} n \rceil$. Every element $j \in [1:n]$ is given a label $L(j)$ which is a string of integers of length $x+1$ drawn from the alphabet $\{0,1,2 \ldots a\}$ of size $a+1$. Let $n= p_d a^d+r_d$ and $n=p_{d-1}a^{d-1}+r_{d-1}$ for some integers $p_d,p_{d-1},r_{d},r_{d-1}$, where $r_d < a^d$ and $r_{d-1}<a^{d-1}$. Now, we describe the sequence of the $d$-th digit across the string labels of all elements from $1$ to $n$: 
  \begin{enumerate}
   \item Repeat $0$ $a^{d-1}$ times, repeat the next integer $1$ $a^{d-1}$ times and so on circularly \footnote{Circular means that after $a-1$ is completed, we start with $0$ again.} from $\{0,1 \ldots a-1 \}$ till $p_da^d$. 
   \item After that, repeat $0$ $\lceil r_d/a \rceil$ times followed by $1$ $\lceil r_d/a \rceil$ times till we reach the $n$th position. Clearly, $n$-th integer in the sequence would not exceed $a-1$. 
   \item Every integer occurring \textit{after} the position $a^{d-1}p_{d-1}$ is increased by $1$.
  \end{enumerate}

 From the three steps used to generate every digit, a straightforward calculation shows that every integer letter is repeated at most $\lceil n/a \rceil$ times in every digit $i$ in the string. Now, we would like to prove inductively that the labels are distinct for all $n$ elements. Let us assume the induction hypothesis: For all $n < a^{q+1}$, the labels are distinct. The base case of $q=0$ is easy to see. Then, we would like to show that for $a^{q+1} \leq n < a^{q+2}$, the labels are distinct. 
    
    Another way of looking at the labeling procedure is as follows. Let $n=a^{q+1}p+r$ with $r< a^{q+1}$. Divide the label matrix $L$ (of dimensions $(q+2) \times n$) into two parts, one $L_1$ consisting of the first $pa^{q+1}$ columns and the other $L_2$ consisting of the remaining columns. The first $q+1$ rows of $L_1$ is nothing but the string labels for all numbers from $0$ to $pa^{q+1}$ expressed in base $a$. For any row $i \leq \lceil \log_a r \rceil$ in the original matrix $L$ of labels, till the end of first $pa^{q+1}$ columns, the labeling procedure would be still in Step $1$. After that, one can take $r$ to be the new size of the set of elements to be labelled and then restart the procedure with this $r$. Therefore we have the following \emph{key} observation: $L_2(1: \lceil \log_a r \rceil,: )$ (the matrix with first $\lceil \log_a r \rceil $ rows of $L_2$) is nothing but the label matrix for $r$ distinct elements from the above labeling procedure. 
    
    Since, $r < a^{q+1}$, by the induction hypothesis, the columns are distinct. Hence, any two columns in $L_2$ are distinct. Suppose the first $q+1$ rows of two columns $b$ and $c$ of $L_1$ are identical. These correspond to base $a$ expansion of $b-1$ and $c-1$. They are separated by at least $a^{q+1}+1$ columns. But the last row of columns $b$ and $p$ in $L_1$ has to be distinct because according to Step $2$ and Step $3$ of the labeling procedure, in the $q+2^{th}$ row, every integer is repeated at most $\lceil n/a \rceil \leq a^{q+1}$ times continuously, and only once. Therefore, any two columns in $L_1$ are distinct. The last row entries in $L_1$ are different from $L_2$ because of the addition in Step $3$. Therefore, all columns of $L$ are distinct. Hence, by induction, the result is shown.

 \subsection{Proof of Theorem \ref{const}}
By Lemma \ref{label}, $i$th place has at most $\lceil \frac{n}{\lceil n/k\rceil} \rceil \leq  k$ occurrences of symbol $j$. Therefore, $\lvert S_{i,j} \rvert \leq k$. Now, consider the pair of distinct elements $p,q \in [1:n]$. Since they are labelled distinctly (Lemma \ref{label}), there is at least one letter $i$ in their string labels where they differ. Suppose the distinct $i$th letters are $a,b \in {\cal A},~a \neq b $ and let us say $a \neq 0$ without loss of generality. Then, clearly the separation criterion is met by the subset $S_{i,a}$. This proves the claim.

 \subsection{Proof of Theorem \ref{equivsep}}
We construct a worst case $\sigma$ inductively. Before every step $m$, the adaptive algorithm deterministically chooses $I_m$ based on $\ess_{\{I_1,I_2 \ldots I_{m-1}\}}(K_n)$. Therefore, we will reveal a partial order $\sigma^{(m-1)}$ to satisfy the observations so far. Inductively for every $m$, we will make sure that after $I_m$ is chosen by the algorithm, further details about $\sigma$ can be revealed to form $\sigma^{(m)}$ such that after intervening on $I_2$ and then applying R$0$, we will make sure there is no opportunity to apply the rule $R2$. This would make sure that ${\cal I}$ is a separating system on $n$ elements.
   
    Before intervention at any step $m$, let us `tag' every vertex $i$ using a subset $C_i^{(m-1)} \subseteq [1:m]$ such that $C_i^{(m-1)} = \{ p : i \in I_p ,~ p \leq m-1 \}$. $C_i^{(m-1)}$ contains indices of all those interventions that contain vertex $i$ before step $m$. Let ${\cal C}^{(m-1)}$ contain distinct elements of the multi-set $\{C_i^{(m-1)}\}$ .We will construct $\sigma$ partially such that it satisfies the following criterion always:
    
     \textit{Inductive Hypothesis:} 
     The partial order $\sigma^{(m-1)}$ is such that for any two elements $i,j$ with $C_i$ and $C_j$, $i$ and $j$ are incomparable if $C_i=C_j$ and comparable otherwise. This means the edges between the elements tagged with the same tag $C$ has not been revealed, and thus the relevant directed edges are not known by the algorithm. 
     
      Now, we briefly digress to argue that if we could construct $\sigma^{(1)}, \sigma^{(2)} \ldots $ satisfying such a property throughout, then clearly all vertices must be tagged differently otherwise the directions among the vertices that are tagged similarly cannot be learned by the algorithm. Therefore, the algorithm has not succeeded in its task. If all vertices are tagged differently, then it means it is a separating system.
      
      \textit{Construction of $\sigma^{(m)}$:} We now construct $\sigma^{(m)}$ that can be shown to satisfy the induction hypothesis before step $m+1$. Before step $m$ , consider the vertices in $C \in {\cal C}^{(m-1)}$ for any $C$. Let the current intervention be $I_m$ chosen by the deterministic algorithm. We make the following changes: Modify $\sigma^{(m-1)}$ such that  vertices in $I_m \bigcap C$ come before $(I_m)^c \bigcap C$ in the partial order $\sigma^{(m)}$ (vertices inside either sets are still not ordered amongst themselves)  in the ordering and clearly the directions between these two sets are revealed by R$0$. By the induction hypothesis for step $m$ and with the new tagging of vertices into ${\cal C}^{(m)}$, it is easy to see that only directions between distinct $C's$ in the new ${\cal C}^{(m)}$ have been revealed and all directions within a tag set $C$ are not revealed and all vertices in a tag set are contiguous in the ordering so far. We need to only show that rule R$2$ cannot reveal anymore edges amongst vertices in $C \in {\cal C}^{(m)}$ after the new $\sigma^{(m)}$ and intervention $I_m$. Suppose there are two vertices $i,j$ such that just after intervention $I_m$ and the modified $\sigma^{(m)}$, they are tagged identically and application of R$2$ reveals the direction between $i$ and $j$ before the next intervention. Then there has to be a vertex $k$ tagged differently from $i,j$ such that $j \rightarrow k$ and $k \rightarrow i$ are both known. But this implies that $j$ and $i$ are comparable in $\sigma^{(m)}$ leading to a contradiction. This implies the hypothesis holds for step $m+1$.
      
      \textit{Base case:} Trivially, the induction hypothesis holds for step $0$ where $\sigma^{(0)}$ leaves the entire set unordered. 
 \subsection{Proof of Lemma \ref{lem:root}}
   The proof is a direct obvious consequence of acyclicity, non-existence of immoralities and the definition of rule R1. 
   
    \subsection{Proof of Lemma \ref{treealg}}
    By Lemma \ref{lem:root}, it is sufficient for an algorithm to identify the root node of the tree. Suppose the root node is $b$ unknown to the algorithm. Every tree has a single vertex separator that partitions the tree into components each of which has size at most $\frac{2}{3}n$ \cite{lipton1979separator}. Choose that vertex separator $a_1$ (it can be found in by removing every node and determining the components left). If it is a root node we stop here. Otherwise, its parent $p_1$ (if it is not) after application of rule R$0$ is identified. Let us consider component trees $T_1,T_2 \ldots T_k$  that result by removing node $a_1$. Let $T_1$ contain $p_1$. All directions in all other trees are known after repeated application of R$1$ on the original tree after R$0$ is applied. Directions in T$1$ will not be known. For the next step, $\ess(T_1)$ is the new skeleton which has no immoralities. Again, we find the best vertex separator $a_2$ and the process continues. This procedure will terminate at some step $j$ when $a_j =b$ or there is only one node left which should be $b$ by Lemma \ref{lem:root}. Since the number of nodes reduce by about $1/3$ at least each time, and initially it can be at most $n$, this procedure terminates in at most $O(\log_2 n)$ steps.

 \subsection{Proof of Lemma \ref{inductree}}
The graph induced by two colors classes in any graph is a bi-partite graph and bi-partite graphs do not have odd induced cycles. Since the graph and any induced subgraph is chordal, it implies the induced graph on a pair of color classes does not have a cycle. This proves the theorem.

  \subsection{Proof of Theorem \ref{thm:infoLower}}
      Assume $n$ is even for simplicity. We define a family of partial order $\sigma^{(p)}$ as follows: Group $i,i+1$ into $C_i$. Ordering among $i$ and $i+1$ is not revealed. But all the edges between $C_i$ and $C_j$ for any $j>i$ are directed from $C_i$ to $C_j$. Now, one has to design a set of interventions such that exactly one node among every $C_i$ is intervened on at least once. This is because, if neither $i$ nor $i+1$ in $C_i$ are intervened on, then the direction between $i$ and $i+1$ cannot be figured out by applying rule R$2$ on any other set of directions in the rest of the graph. Since the size of every intervention is at most $k$ and at least $n/2$ nodes need to be covered by intervention sets, the number of interventions required is at least $\frac{n}{2k}$.

  \subsection{Proof of Theorem \ref{thm:randomized}}
\begin{proof}
Separate $n$ vertices arbitrarily into $\frac{n}{k}$ disjoint subsets $C_i$ of size-$k$. Let the first $n/k$ interventions $\{I_1,I_2,...,I_{n/k}\}$ be such that $I_i(v)=1$ if and only if $v\in C_i$. This divides the problem of learning a clique of size $n$ into learning $n/k$ cliques of size $k$. Then, we can apply the clique learning algorithm in \cite{Hu2014} as a black box to each of the $\frac{n}{k}$ blocks: Each block is learned with probability $k^{-c}$ after $\log c\log k$ experiments in expectation. For $k=cn^r$, choose $c>1/r-1$. Then the union bound over $n/k$ blocks yields probability polynomial in $n$. Since each block takes $\mathcal{O}(\log\log k)$ experiments, we need $\frac{n}{k}\mathcal{O}(\log\log k)$ experiments. 
\end{proof}

\subsection{Proof of Theorem \ref{lowbnd}}
We need the following definitions and some results before proving the theorem.
\begin{definition}
  A perfect elimination ordering $\sigma_p=\{v_1,v_2 \ldots v_n\}$ on the vertices of an undirected chordal graph $G$ is such that for all $i$, the induced neighborhood of $v_i$ on the subgraph formed by $\{v_1,v_2 \ldots v_{i-1} \}$ is a clique.
\end{definition}

\begin{lemma}\label{buhl1}
(\cite{hauser2014two}) If all directions in the chordal graph are according to perfect elimination ordering (edges go only from vertices lower in the order to higher in the order), then there are no immoralities.
\end{lemma}

We make the following observation:  Let the directions in a graph $D$ be oriented according to an ordering $\sigma$ on the vertices. If a clique comes first in the ordering, then the knowledge of edge directions in the rest of the graph, excluding that of the clique, cannot help at any stage of the intervention process on the clique; because all the edges are directed outwards from the clique and hence none of the Meek rules apply. This is because, if $a \rightarrow b$ is to be inferred by Meek rules from other known directions, then either there has to be a known edge direction into $a$ or $b$ before the inference step. So if one of the directed edges not from the clique was to help in the discovery process, either that edge has to be directed towards $a$ or $b$ (like in Meek rules R$1$, R$2$ and R$3$), or it has to be directed towards $c$ in another $c \rightarrow a$ (R$4$) which belongs to the clique. Both the above cases are not possible.

\begin{lemma}\label{buhl2}
 (\cite{hauser2014two}) Let $C$ be a maximum clique of an undirected chordal graph $\ess(D)$, then there is an underlying DAG $D$ on the chordal skeleton that is oriented according to a perfect elimination ordering (implying no immoralities), where the clique $C$ occurs first.
\end{lemma}

 By Lemmas \ref{buhl1}, \ref{buhl2} and the observation above, given a chordal skeleton, we can construct a DAG on the skeleton with no immoralities such that the directions of the maximum clique in $D$ cannot be learned by using knowledge of the directions outside. This means that only the intervention sets $\{I_1 \bigcap C, I_2 \bigcap C  \ldots\}$ matter for learning the directions on this clique. Therefore inference on the clique is isolated. Hence, all the lower bounds for the clique case transfer to this case and since the size of the largest clique is exactly the coloring number of the chordal skeleton, the theorem follows.

 \subsection{Proof of Theorem \ref{lower-bound}} 
  \textit{Example with a feasible solution with $\lvert {\cal I} \rvert$ close to the lower bound:}
    Consider a graph $G$ that can be partitioned into a clique of size $\chi$ and an independent set $\alpha$. Such graphs are called split graphs and as $n \rightarrow \infty$, the fraction of split graphs to chordal graphs tends to $1$. 
   If $\ess(D)=G$ where $G$ is a split graph skeleton, it is enough to intervene only on the nodes in the clique and therefore the number of interventions that are needed is that for the clique. It is certainly possible to orient the edges in such a way so as to avoid immoralities, since the graph is chordal.

   \textit{Example with $\lvert {\cal I} \rvert$ which needs to be close to the upper bound:}
    We construct a connected chordal skeleton with independent set $\alpha$ and clique size $\chi$ (also coloring number) such that it would require $\frac{\alpha  (\chi-1)}{2k}$ interventions at least for any algorithm over a class of orientations. 
    
    Consider a line $L$ consisting of vertices $1,2 \ldots 2 \alpha$ such that every node $1<i<2\alpha$ is connected to $i-1$ and $i+1$. For, all $1 \leq p \leq \alpha$, consider a clique $C_p$ of size $\chi$ which only has nodes $2p-1,2p$ from the line $L$. Now assume that the actual orientation of the L is  $1 \rightarrow 2 \ldots \rightarrow 2 \alpha$. In every clique, the orientation is partially specified as follows: In every clique $C_p$, all edges from node $2p-1$ are outgoing. It is very clear that this partial orientation excludes all immoralities. Further, each clique $C_p-\{2p-1\}$ can have any arbitrary orientation out of $\chi-1$ possible ones in the actual DAG. Now, even if all the specified directions are revealed to the algorithm, the algorithm has to intervene on all $\alpha$ disjoint cliques $\{C_p-\{2p-1\}\}_{p=1}^{\alpha}$ each of size $\chi-1$ and directions in one clique will not force directions on the others through any of the Meek rules or rule R$0$. Therefore, the lower bound of $\frac{\alpha(\chi-1)}{2}$ total node accesses (total number of nodes intervened) is implied by Theorem \ref{thm:infoLower}. Given every intervention is of size $k$, these chordal skeletons with the revealed partial order needs at least $\frac{\alpha(\chi-1)}{2k}$ more experiments.

\subsection{Performance Comparison of Our Algorithm vs. Naive Scheme for $n=500,k=10$ and $n=2000,k=20$}
\begin{figure*}[ht!]
\centering
\subfigure[$n=500,k=10$]{\label{fig:n500k10}\includegraphics[width=0.45\columnwidth]{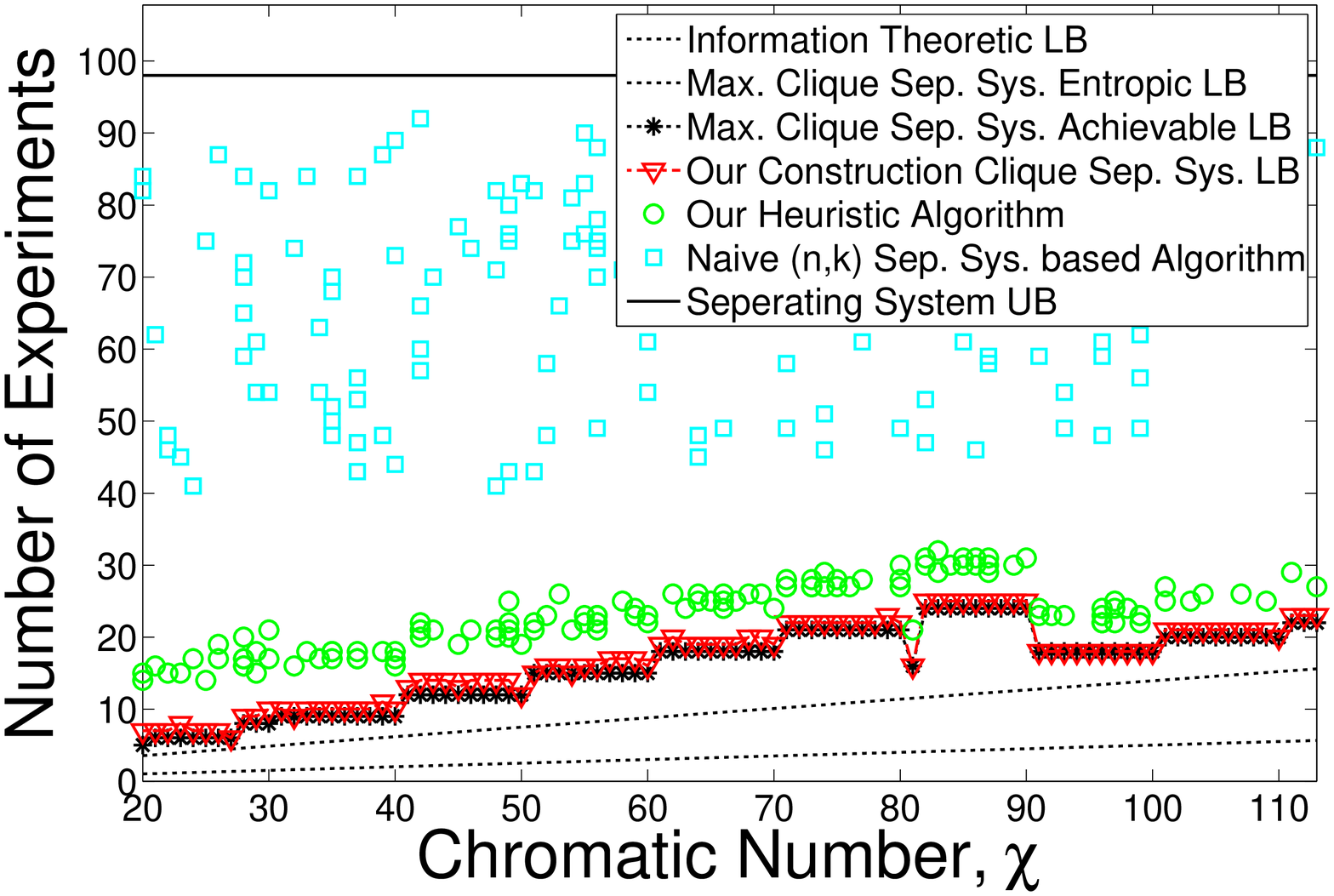}}
\subfigure[$n=2000,k=20$]{\label{fig:n2000k20}\includegraphics[width=0.45\columnwidth]{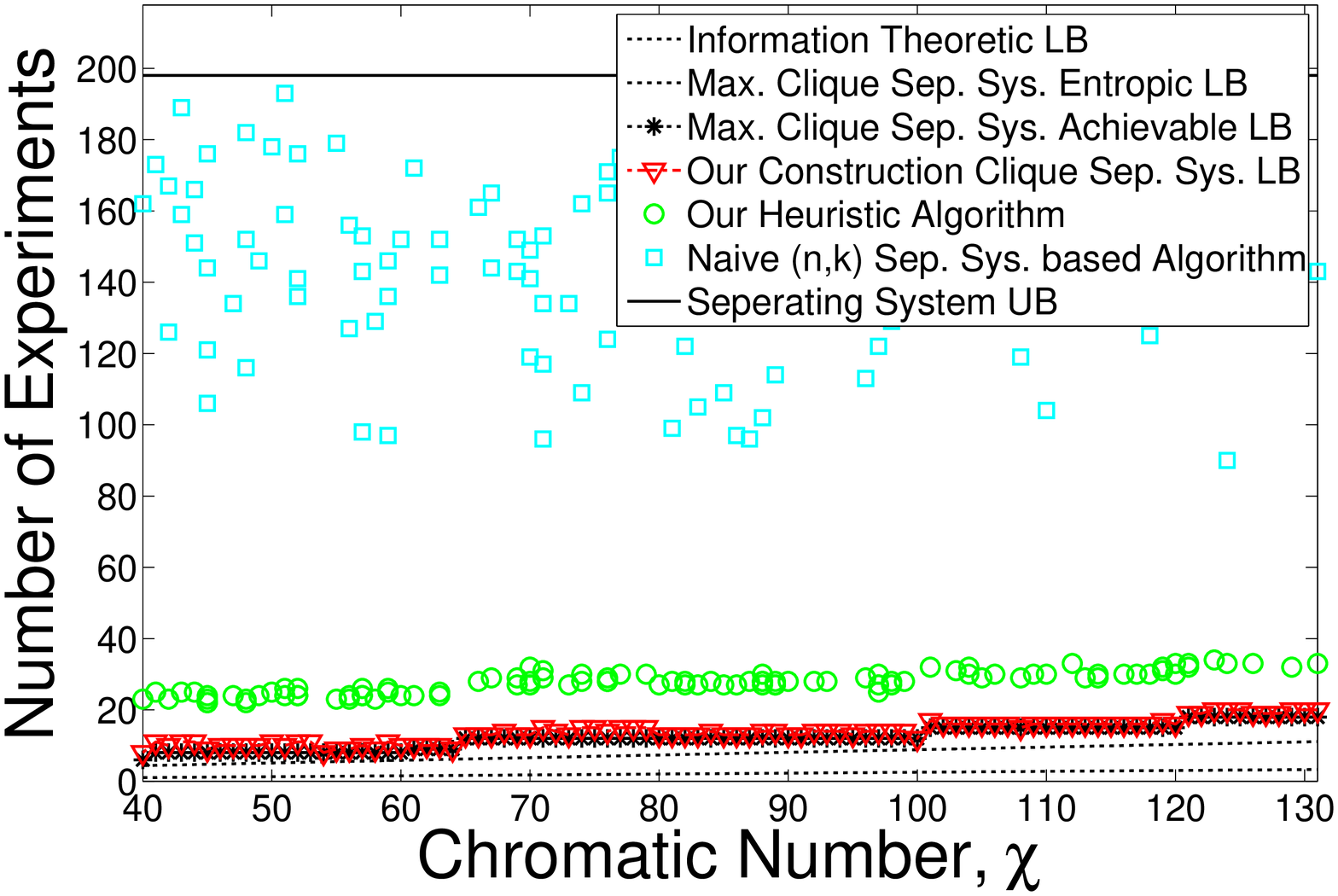}}
\caption{$n$: no. of vertices, $k$: Intervention size bound. The number of experiments is compared between our heuristic and the naive algorithm based on the $(n,k)$ separating system on random chordal graphs. The red markers represent the sizes of $(\chi,k)$ separating system. Green circle markers and the cyan square markers for the same $\chi$ value correspond to the number of experiments required by our heuristic and the algorithm based on an $(n,k)$ separating system(Theorem \ref{const}), respectively, on the same set of chordal graphs. All four plots (including the ones in the main text) indicate that our algorithm requires number of experiments proportional to the clique number $\chi$, whereas naive separating system based algorithm requires experiments on the order of number of variables $n$.
}
\label{fig:simulation2}
\end{figure*}

\subsection{Proof of Theorem \ref{Correctness}}
We provide the following justifications for the correctness of Algorithm \ref{alg:hybrid}. 
  \begin{enumerate}
    \item At line \ref{linecolor} of the algorithm, when Meek rules and R$0$ are applied after every intervention, the intermediate graph $G$, with unlearned edges, will always be a disjoint union of chordal components (refer to (\ref{eqn:learnseq}) and the comments below) and hence a chordal graph.   
    \item  The number of unlearned edges before and after the main while loop in Algorithm \ref{alg:hybrid} reduces by at least one. Every edge in $E$ is incident on two colors and one of the colors is always picked for processing because we use a separating system on the colors. Therefore, one node belonging to some edge has a positive score and is intervened on. The edge direction is learnt through rule R$0$. Therefore, the algorithm terminates. 
    \item It identifies the correct $\vec{G}$ because every edge is inferred after some intervention $I_t$ by applying rule R$0$ and Meek rules as in (\ref{eqn:learnseq}) both of which are correct.
    \item the algorithm has polynomial run time complexity because the main while loop ends in time $\lvert E \rvert$.
  \end{enumerate}
  \newpage
  
\small{
}

\end{document}